\documentclass[oneside,11pt]{article}

\usepackage{arxiv}
\usepackage[utf8]{inputenc}
\usepackage[T1]{fontenc}
\usepackage{lmodern}
\usepackage{amsmath,a4wide,epsfig,psfrag}
\usepackage{amssymb}
\usepackage{amsthm}
\usepackage{pst-all}
\usepackage{graphicx}
\graphicspath{ {./pics/} {./eps/}}
\usepackage{epsfig}
\usepackage{listings}
\usepackage{color}
\usepackage{multirow}
\usepackage{mathtools}
\usepackage[english]{babel}
\usepackage[overload]{empheq}
\usepackage{caption}
\usepackage{subcaption}
\usepackage{float} 
\newtheorem{proposition}{Proposition}
\newtheorem{theorem}{Theorem}
\newtheorem{example}{Example}
\newtheorem{lemma}{Lemma}
\newtheorem{remark}{Remark}
\newtheorem{corollary}{Corollary}
\newtheorem{definition}{Definition}

\newcommand{\R}{\mathbb{R}}

\begin{document}
	
	\title{On decision regions of narrow deep neural networks}
	\author{
 	Hans-Peter Beise\\
  	Department of Computer Science\\
 	Trier University of Applied Sciences\\
 	\And
 	Steve Dias Da Cruz \\
  	IEE S.A.\\ 	 
  	University of Kaiserslautern
	\And
  	Udo Schr\"oder\\
  	IEE S.A.\\
	}

	
\maketitle      
\begin{abstract}
	We show that for neural network functions that have width less or equal to the input dimension all connected components of decision regions are unbounded. The result holds for continuous and strictly monotonic activation functions as well as for the ReLU activation function. This complements recent results on approximation capabilities by \cite{hanin2017approximating} and connectivity of decision regions by \cite{nguyen2018neural} for such narrow neural networks. Our results are illustrated by means of numerical experiments.
\end{abstract}


\section{Introduction}
In recent years machine learning experienced a remarkable evolution mainly due to the progress achieved with deep neural networks, c.f. \cite{krizhevsky2012imagenet}, \cite{hinton2012deep}, \cite{nguyen2018neural}, \cite{he2016deep},\cite{schmidhuber2015deep} and \cite{goodfellow2016deep} for an overview and theoretical background. In the course of this astonishing success in applications, there has been huge progress in the research towards understanding the mathematical properties of neural network functions. As  part of this, the approximation properties, or expressiveness, of neural network functions have attracted intense interest in recent research. The central result in this field is the classic \emph{universal approximation theorem}, which states that any continuous function can be approximated with arbitrary accuracy (in terms of uniform approximation or $L_p$ norms) by neural network functions that have only one hidden layer for nearly every activation function c.f. \cite{cybenko1989approximation}, \cite{hornik1991approximation}, see also \cite{scarselli1998universal} for an overview of classical results and \cite{gnecco2011some} for a comparison with other common approximation techniques. In last years, this kind of results where further improved in several respects \cite{costarelli2013approximation,petersen2018optimal,yarotsky2017error,yarotsky2018optimal, guliyev2018approximation,elfwing2018sigmoid,bolcskei2019optimal,petersen2020equivalence,guhring2020error} for some recent results in this direction.\\
On the other hand, from empirical observations it turned out that depth has a significant impact on the performance of neural networks which is why a lot of research has been dedicated to to analyse the effect of depth on the expressive power of neural networks, c.f. \cite{telgarsky2016benefits,mhaskar2016deep,montufar2014number,raghu2017expressive,rolnick2017power,lin2017does,cohen2016expressive}. \\
It is however clear that besides depth a neural network needs a certain width in order to be a universal approximator, c.f. Remark \ref{remarkLowerTrivial}. The role of width with regard to approximation properties is investigated in a line of works \cite{lu2017expressive, hanin2017approximating,hanin2017universal,johnson2018deep,kidger2020universal,park2020minimum}, and the importance of width from the perspective of decision regions is pointed out in \cite{nguyen2017loss}. In this work, we follow the latter work and investigate the expressiveness of networks of bounded width mainly in the latter terms.


Let us introduce the following notation. By $\lVert\cdot \rVert$ we denote the Euclidean norm in $\R^d$. For a set $D\subset \R^d$ we denote by $D^\circ$ the set of interior points, by $\overline{D}$ its closure, by $\partial D=\overline{D}\setminus D^\circ$ the boundary of $D$ and for $f:\R^d\rightarrow \R^m$ we set $\left\Vert f\right\Vert_D:=\sup\{\left\Vert f(x)\right\Vert :x\in D\}$. For $W\in \R^{m\times n}$ we denote by $\lVert W\rVert_{\textnormal{op}}:=\max_{\lVert x \rVert\leq 1 }\lVert Wx \rVert$ the operator norm. We consider neural network functions $F:\R^{d_{in}}\rightarrow \R^{d_{out}}$ where $d_{in}$ is called the input dimesion and $d_{out}$ the output dimension. Our network functions have the following form $F:=W_L\circ A_{L-1}\circ ... \circ A_1$ where $A_j(x)=\sigma(W_jx+b_j)$ with $W_j\in \R^{d_{j}\times d_{j-1}}$ (weights), $b_j\in \R^{d_j}$ (bias) and $\sigma:\R \rightarrow \R$ the activation function. We emphasize that $\sigma$ and the preimage $\sigma^{-1}$ are understood to be applied elementwise when applied to vectors or subsets of $\R^d$. The widely applied activation function, called \emph{rectified linear unit}, or shortly ReLU, is defined by $t\mapsto\max\{t,0\}$.
We set $d_0:=d_{in}$ and $d_L=d_{out}$ and call $d_j$  the \emph{width} of layer $j=1,...,L$. The width of the network is defined as $\omega:=\omega_F=\max\{d_j: \, j=1,...,L\}$ and $L$ is called the depth of the network. Adapting the notation from \cite{hanin2017approximating} to our needs, we define $\omega_{\min}(\sigma, d_{in},d_{out})$ to be the minimum width such that for every continuous $f:[a,b]^{d_{in}}\rightarrow \R^{d_{out}}$ and $\varepsilon>0$ there exists a network function $F:\R^{d_{in}}\rightarrow \R^{d_{out}}$ with activation function $\sigma$ and $\omega_F\leq\omega_{\min}(\sigma, d_{in},d_{out})$ such that $\left\Vert f-F\right\Vert_{[a,b]^{d_{in}}}<\varepsilon$, where $a<b$ are some real numbers.

\section{Related work}
It is a fundamental observation that the expressiveness of a function that implements a classifier model is closely related to the notion of decision regions. 
\begin{definition}\label{decRegion}
	For a network function $F=(F_1,...,F_{d_{out}}):\R^{d_{in}}\rightarrow \R^{d_{out}}$ and $j\in\{1,...,d_{out}\}$, the set $C_j:=\{x\in\R^{d_{in}}: F_j(x)>F_k(x),\textnormal{for all } k\neq j\}$ is called a decision region (for class $j$). If $K\subset\R^{d_{in}}$, then $C_j\cap K$ is called the decision region (of class $j$) in $K$. 
\end{definition} 
\begin{definition}\label{innerConnect}
	A set $C\subset \R^d$ is said to be connected if there exist no disjoint open sets $U,\, V  \subset \R^d$ such that $C\subset U\cup V$ and $C\cap U$ and $C\cap V$ are non-empty.
	Let $K\subset\R^d$ be some compact set, then $C$ is said to be connected in $K$, if such sets $U,\,V$ do not exist for $ C\cap K$.
\end{definition}

It is known that open sets can be decomposed into a disjoint family of connected open sets. In what follows, an element of this family is called
connected \emph{component}. These components are maximal in the sense that they are no proper subset of another connected subset of the original set.\\
It should be noted that the notion of connectivity is not equivalent to the more intuitive but more restrictive term of path-connectivity (c.f. Definition \ref{innerPathConnect}). 


\begin{definition}\label{innerPathConnect}
	A set $C\subset \R^d$ is said to be path-connected if for all  $x_1,\, x_2\in C$ there exists a continuous path $\gamma:[0,1]\rightarrow C$ such that $\gamma(0)=x_1$  and $\gamma(1)=x_2$.
	Let $K\subset\R^d$ be some compact set, then $C$ is said to be path-connected in $K$, if for all $x_1,\, x_2\in C\cap K$
	there exists a continuous path $\gamma:[0,1]\rightarrow C\cap K$ such that $\gamma(0)=x_1$ and $\gamma(1)=x_2$.	
\end{definition}

In \cite{fawzi2017classification} the decision regions of deep neural networks are investigated empirically. By experiments with ImageNet networks, the authors observe that two samples that are predicted to belong to the same class can be connected by a continuous path, where the path is found by a dedicated algorithm that is also provided in \cite{fawzi2017classification}. The latter article further gives some interesting insight regarding the local curvature of decision boundaries.

In contrast to the experimentally driven work, our work is more related to \cite{nguyen2018neural} which treats this problem from a theoretical perspective.  A central result therein is the following.
\begin{theorem}[Theorem 3.10,\cite{nguyen2018neural}]\label{theHeinConnected}
	Let $F$ be a neural network function such that $d_{in}=d_1\geq d_2\geq ...d_L=d_{out}$ and each weight matrix has full rank. If the activation function $\sigma$ is continuous, strictly monotonically increasing and satisfies $\sigma(\R)=\R$, then every decision region is connected.	
\end{theorem}

In parallel to this, the approximation properties of width bounded neural networks have been studied in several works, \cite{lu2017expressive, hanin2017approximating,hanin2017universal,johnson2018deep,kidger2020universal,park2020minimum}. A common goal in these lines of research is to provide upper and lower bounds on the minimum width required to ensure universal approximation in $C(K,\R^d_{out})$, i.e the space of continuous function from a compact set $K\subset\R^{d_{in}}$ to $\R^{d_{out}}$, and in $L_p(\R^{d_{in}},\R^{d_{out}})$ or $L_p(K,\R^{d_{out}})$ ($K$ again a compact set in $\R^{d_{in}})$.  
 For the case of $C(K,\R^{d_{out}})$, it is proven in \cite{hanin2017approximating} that 

\begin{equation}\label{haninInequ}
d_{in}+1 \leq\omega_{\min}(\textnormal{ReLU},d_{in},d_{out}) \leq d_{in}+d_{out}.
\end{equation}

Further, in \cite{johnson2018deep} it is shown that 
\begin{equation}\label{ourEstimate}
\omega_{\min}(\sigma,d_{in},d_{out}) \geq d_{in}+1,
\end{equation}
for activation functions  $\sigma$ that admit arbitrary accurate uniform approximation by injective continuous functions on arbitrary compact subset of $\R$. 

Recently, results as stated in (\ref{haninInequ}) and (\ref{ourEstimate}), have been extended and partially improved in \cite{kidger2020universal,park2020minimum}. We refer to the latter work for a good overview. 


In this work, we take a similar focus as in \cite{nguyen2018neural}, namely limitations regarding the decision regions of neural networks of bounded width. Our main result complements the results on the lower bound found in \cite{hanin2017approximating}, c.f. (\ref{haninInequ}), and a result from \cite{nguyen2018neural}, c.f. (\ref{theHeinConnected}), in terms of decision regions. We show that for network functions with maximum width $d_{in}$ and both types of activation functions, strictly monotonic or ReLU, the decision regions do not need to be connected, as opposed to \cite{nguyen2018neural}, c.f. (\ref{theHeinConnected}), but are unbounded and therefore do not admit arbitrary accurate approximation of all continuous functions on compact sets. 

The impossibility of universal approximation in the latter case is also obtained in \cite{johnson2018deep}, c.f. $(\ref{ourEstimate})$. To this end, it is shown in \cite[Lemma 4]{johnson2018deep} that the contour lines of network functions $F:\R^{d_{in}}\rightarrow \R$ are unbounded. Our main result extends this by giving a related statement for decision regions of vector valued network functions.

\section{Results} 
In our main result we show that in case of $\omega_F\leq d_{in}$ and continuous and strictly monotonic activation functions or ReLU activation the components of the decision regions are unbounded. This implies that they intersect the boundary of the natural bounding box of input data. 
%

To this end, we exploit the basic observation of the following lemma and show that for certain narrow neural networks this can be continued to the input domain.
The content of the lemma is well-known. We give a short proof for interested readers.
\begin{lemma}\label{lemHypersurface}
	Let $A\in \R^{n\times d}$ with $n\leq d$ , $b\in\R^n$ and $x\in\R^d$ with $A x< b $. Then there exists a non-zero $v\in\R^d$ such that 
	$A(x+\lambda v)< b$ for all $\lambda\geq 0$
\end{lemma}

The geometrical interpretation of the previous lemma is that a convex set in $\R^d$, that is described by less than $d+1$ hypersurfaces, cannot enclose a point. It directly follows that for every compact set $K$ such that if
$C=\{x\in K: A x< b\}\neq \emptyset$ then $C\cap \partial K\neq \emptyset$. 

\begin{proof}
	In case where $A$ is invertible let $w$ be the unique solution of $Aw=b$. One directly verifies that $v=x-w$ has the desired property. Otherwise, we set $v$ equal to one of the (non-zero) vectors that are orthogonal to the rows of $A$.
\end{proof}

In what follows, it will be convenient to argue with invertible weight matrices. The following remark clarifies that this is justified in our setting.

\begin{lemma}\label{lemApproxOneLayer}
	Let $W\in\R^{d\times d }$, $K\subset \R^d$ be a compact set and $\phi,\tilde{\phi}:K\rightarrow \R^d$ be continuous mappings such that for some $\varepsilon_1>0$ 
	\[\lVert \tilde{\phi}-\phi\rVert_K<\varepsilon_1.\]
	 Then for every $\varepsilon_2>\lVert W\rVert_{\textnormal{op}} \,\varepsilon_1$, there exists an invertible $\tilde{W}\in\R^{d\times d}$ such that 
	\[\lVert \tilde{W} \tilde{\phi}-W\phi\rVert_K<\varepsilon_2.\]
\end{lemma}

\begin{proof}
For $\delta=\varepsilon_2-\lVert W\rVert_{\textnormal{op}}\, \varepsilon_1$ we can find an invertible $\tilde{W}\in\R^{d\times d}$ such that 
\[\Vert \tilde{W}\tilde{\phi}-W\tilde{\phi}\rVert_K<\delta.\]
With the triangle inequality and elementary properties of the operator norm it follows that for every $x\in K$ we have
\[\lVert\tilde{W}\tilde{\phi}(x)-W\phi(x)\rVert\leq \lVert\tilde{W}\tilde{\phi}(x)-W\tilde{\phi}(x)\rVert+
\lVert W\rVert_{\textnormal{op}}\, \lVert \tilde{\phi}(x)-\phi(x)\rVert<\varepsilon_2.\] This concludes the proof.

\end{proof}

\begin{lemma}\label{lemApproxNN}
	Let $F:\R^{d}\rightarrow \R^d$ be a neural network function given by $F:=A_{L}\circ ... \circ A_1$ where $A_j(x)=\sigma(W_jx+b_j)$ with $W_j\in \R^{d\times d}$ for $j=1,...,L$, and with a continuous activation function $\sigma$.
For a compact $K\subset \R^d$ and $\varepsilon>0$, there exist invertible $\tilde{W_j} \in \R^{d\times d},\, j=1,...,L$ such that the network function $\tilde{F}:\R^d\rightarrow	\R^d$ defined by $\tilde{F}:=\tilde{A}_{L}\circ ... \circ \tilde{A_1}$ with $\tilde{A_j}(x)=\sigma(\tilde{W_j}x+b_j)$, $j=1,...,L$, approximates $F$ in a way that
	\[\lVert \tilde{F}-F\rVert_K<\varepsilon.\]

\end{lemma}

\begin{proof}
We set $M:=\max\{1,\lVert W_1\rVert_{\textnormal{op}},\lVert W_2\rVert_{\textnormal{op}},...,\lVert W_L\rVert_{\textnormal{op}}\}$.
We want to use the uniform continuity of $\sigma$ on compact sets. To this end let $Q_1,Q_2,...,Q_{L}\subset\R^d$, $Q_0:=K$ be compact sets iteratively defined by 
\[Q_j:=\bigcup\limits_{x\in Q_{j-1} }\left(\{y:\lVert W_jx+b_j-y\rVert\leq 1\}\cup \{y:\lVert \sigma(W_jx+b_j)-y\rVert\leq 1\}\right), \ j=1,...,L-1\]
and $Q:=\bigcup_{j=0}^{L-1} Q_j$. 
This yields a sufficiently large compact set that will contain the output of every intermediate layer, with and without activation,  of the network function that will turn out from our below construction for every input $x\in K$.
As a component wise applied function from $\R^d$ to $\R^d$,  $\sigma$ is uniformly continuous on $Q$ and we hence find $1>\delta_{L-1},...,\delta_1,\delta_0>0$ corresponding to $\varepsilon=:\delta_L$ such that $\delta_j>M\delta_{j-1}$, $j=1,...,L$, and for all $x,y\in Q$ with $\lVert x-y\rVert<M \delta_j $ the following estimate holds
\begin{equation}\label{equConEq}
\lVert\sigma(x)-\sigma(y)\rVert<\delta_{j+1},\quad j=1,...,L-1.
\end{equation}
Now, one can iteratively apply Lemma \ref{lemApproxOneLayer} to yield the desired $\tilde{W_j}$.
In fact, assume that we have already determined invertible $\tilde{W}_1,...,\tilde{W}_{j-1}\in\R^{d\times d}$ such that the corresponding $\tilde{A}_{k}(x)=\sigma(\tilde{W_k}x+b_k)$, $k=1,...,j-1$ satisfy
\[\lVert  \tilde{A}_{j-1}\circ...\tilde{A}_2\circ \tilde{A}_1-A_{j-1}\circ ...
A_2\circ A_1\rVert_K<\delta_{j-1}\] 
and $\tilde{A}_{j-1}\circ...\tilde{A}_2\circ \tilde{A}_1(x)\in Q_{j-1}$. 
Then Lemma \ref{lemApproxOneLayer} delivers an invertible $\tilde{W_j}\in \R^{d\times d}$ such that
\[\lVert \tilde{W_j}-W_j \rVert_{Q_{j-1}}<M\delta_{j-1},
\]
where the matrices in the norms of the latter inequality are interpreted as the corresponding linear mappings. 
The application of (\ref{equConEq}) to $\tilde{A}_{j}:=\sigma(\tilde{W_j}x+b_j)$ yields
\[\lVert  \tilde{A}_{j}\circ...\tilde{A}_2\circ \tilde{A}_1-A_{j}\circ ...
A_2\circ A_1\rVert_K<\delta_{j}.\] 
From the construction of $Q_j$ and $\delta_j<1$ it further follows that $\tilde{A_j}(x)\in Q_j$ for all $x\in Q_{j-1}$. This concludes the general step and the assertion follows inductively.
\end{proof}

\begin{proposition}\label{propAppNN}
Let $F:\R^{d_{in}}\rightarrow \R^{d_{out}}$ be a neural network function $F(x)=W_L (A_{L-1}\circ ... \circ A_1(x))+b_L$ with $\omega_F\leq d_{in}$ and $L$ layers. Then for a given $\varepsilon>0$, there exist invertible $\tilde{W}_j\in \R^{d_{in} \times d_{in}}$, $j=1,...,L-1$ such that $\tilde{F}(x)=W_L(\tilde{A}_{L-1}\circ ... \circ \tilde{A}_1(x))+b_L$ with $\tilde{A_j}(x)=\sigma(\tilde{W_j}x+b_j)$, $j=1,...,L-1$ satisfies 
\[\lVert \tilde{F}-F\rVert_K<\varepsilon.\]
\end{proposition}

\begin{proof}
By padding with zeros rows and zero components, we can consider $x\mapsto W_jx+b_j$, $j=1,...,L-1$ as mapping from $\R^{d_{in}}$ to $\R^{d_{in}}$ so that $W_j\in \R^{d_{in} \times d_{in}}$ and $b_j\in \R^{d_{in}}$, $j=1,...,L-1$. Hence, since the final layer is the same for both $F$ and $\tilde{F}$, the result follows immediately from Lemma \ref{lemApproxNN} applied to the first $L-1$ layers.
\end{proof}

\begin{remark}\label{appWithDecRegion}
Let $F$ be a neural network function as in Proposition \ref{propAppNN} and $C$ a fixed connected component of a decision region of $F$, say $C$ a component $C_1$, and $K$ a compact subset of $\R^{d_{in}}$ that has non-empty intersection with $C$. Let further  $\varepsilon>0$. Then by means of Proposition $\ref{propAppNN}$ we find a network function $\tilde{F}$ with invertible square weight matrices in the first $L-1$ layers such that
\[\lVert \tilde{F}-F\rVert_K<\varepsilon\]
and further the decision region of $\widetilde{F}$ corresponding to the first class has a connected component $\widetilde{C}$ such that $\widetilde{C}\cap K\subset C\cap K$.

In fact, first one selects $\varepsilon=\varepsilon(C)>0$ so small that a network function that approximates $F$ by $\varepsilon$ with respect to norm $\lVert \cdot \rVert_K$, automatically has a decision region corresponding to the first class that intersects $C\cap K$. Then one approximates $F$ with an accuracy of $\varepsilon/2$ by a network $H$ by means of Proposition \ref{propAppNN}. In the network function $H$ one decreases the first entry in the bias vector of the final layer $L$ by $\varepsilon/2$ to give the desired approximating network function $\tilde{F}$. Then by the triangle inequalitiy
\[\lVert \tilde{F}(x)-F(x)\rVert <\varepsilon.\]
Further, the first components $F_1,\,H_1,\,\tilde{F}_1$ of $F,\, H,\,\tilde{F}$, respectively, satisfy 
\[\tilde{F}_1(x) < \tilde{F}_1(x)+\varepsilon/2 -\left(H_1(x)-F_1(x)\right)  = F_1(x) \]
for every $x\in K$.
\end{remark}

Our main result now states as follows.

\begin{theorem}\label{theoDecRIntersect}
	Let $F:\R^{d_{in}}:\rightarrow \R^{d_{out}}$ be a neural network function with continuous and strictly monotonic activation function $\sigma$ or $\sigma=$ReLU and $\omega_F\leq d_{in}$. Then for every decision region $C_j$, $j=1,...,d_{out}$, the connected components of $C_j$ are unbounded.
\end{theorem}

\begin{proof}
	For a general $d_{in}$-dimensional box $K=[a,b]^{d_{in}}$, $a<b$, we show that each connected component of the decision regions intersect $\partial K$. We consider $C_1$, the decision region for the first class (c.f. Definition \ref{decRegion}), for which we fix a connected component that intersects $K$. We denote the intersection of this component with $K$ by $C$. That is, $C$ is non-empty and open in $K$ and $C\cap K^\circ\neq\emptyset$. We assume that each weight matrix $W_j$, $j=1,...,L-1$ is in $\R^{d_{in}\times d_{in}}$ and invertible. By Proposition \ref{propAppNN} and Remark \ref{appWithDecRegion}, this covers the remaining cases.
	By the definition of $C_1$ and by continuity we have that $\widetilde{C}:=A_{L-1}\circ ...\circ A_1(C)$ is a connected subset of $\Omega=\{y\in \R^{d_{in}}:b_{1,L}-b_{k,L}>(w_{k,L}-w_{1,L})y, \ k=2,...,d_{out}\}$ where $w_{j,L}$ denotes the $j-$th row of $W_{L}$ and $b_{j,L}$ is the $j-$th component of $b_L$.

 Lemma \ref{lemHypersurface} yields that $\Omega$ is unbounded, since otherwise the hypersurfaces defined by $b_{1,L}-b_{k,L}=(w_{k,L}-w_{1,L})y$ for $k=2,...,d_{out}$ would enclose the points in $\Omega$.
 \\
First, consider the case when $\sigma$ is continuous and strictly monotonic. 
In this case, our assumptions give that $A_{L-1}\circ ...\circ A_1$ is injective. As it is the image of a non-empty bounded set under a continuous mapping defined on the whole domain $\R^{d_{in}}$, $\widetilde{C}$ is a non-empty bounded set. Since $\Omega$ is unbounded, as it is pointed out above, there exists an  $y_0\in \partial\widetilde{C}\cap \Omega$. The compactness of $K$ and the fact that $A_{L-1}\circ ...\circ A_1(C)\subset A_{L-1}\circ ...\circ A_1(K)$ ensures the existence of an $x_0\in K$ with $F(x_0)=y_0$. Further, the Invariance Domain Theorem (also known as Brouwer Invariance Domain Theorem), which applies to $A_{L-1}\circ ...\circ A_1$ in this case, implies that inner points of $C$ are mapped to inner points and hence $x_0$ must be an element of $\partial C$. 

Then, if $x_0\in \partial K$ the proof for this case is finished. Otherwise $x_0$ is an interior point in $K$. Since $y_0$ is an interior point of $\Omega$, the continuity of $A_{L-1}\circ ...\circ A_1$ implies the existence of a small $\varepsilon>0$ such that $B=\{x:\left\Vert x-x_0\right\Vert_2 <\varepsilon \}$ is a subset of $K$ and such that $A_{L-1}\circ ...\circ A_1(B)\subset \Omega$. Hence, $C\cup B\subset C_1$, since by definition $C_1=(A_{L-1}\circ ...\circ A_1)^{-1}(\Omega)$. But $B$ is not a subset of $C$ which contradicts the fact that $C$ is a connected component of $C_1$.

Now the case  $\sigma=$ReLu is considered. If for all $x\in C$, we have $W_j A_{j-1}\circ ...\circ A_1(x)+b_j>0$ for $j=2,...,L-1$, then $A_{L-1}\circ ...\circ A_1$ constitutes a linear affine and invertible map from $C$ to $\widetilde{C}$. Thus, following the same line of arguments as in the previous case, we obtain that $\partial K\cap C\neq\emptyset$. Otherwise, there exist an $x_0\in C$ and a smallest $ l\in\{1,..., L-1\}$ with corresponding $k\in\{1,...,d_{in}\}$ such that $w_{k,l}y_0+b_{k,l}\leq 0$ where $y_0:=A_{l-1}\circ ...\circ A_1(x_0)$ if $l>1$ and $y_0=x_0$ otherwise, and where $w_{k,l}$ denotes the $l-$th row of $W_{l}$ and $b_{k,l}$ is the $l-$th component of $b_l$. Without loss of generality say $k=1$.
	Then by the definition of ReLU, the classification does not change on $y_t=y_0-t e_1$, $t>0$ where $e_1=(1,0,...,0)^T$. More precisely, $A_{L-1}\circ ...\circ A_l(\sigma (y_0))=A_{L-1}\circ ...\circ A_l(\sigma (y_t))\in \Omega$ for all $t>0$. Since we have chosen $l$ to be minimal, the mapping $A_{l-1}\circ ...\circ A_1$, for $l>1$ or identitiy otherwise, is linear affine and invertible on $C$. Its preimage of the half line $\{y_0-t e_1:t\geq 0\}$ thus intersects $\partial C$ in some point $w$. Since by the preceding consideration $A_{L-1}\circ ...\circ A_1(w)\in \Omega$, we can conclude that $w\in \partial K$ by the same arguments as in the cases above.
	
\end{proof}

With Theorem \ref{theoDecRIntersect} we can now extend the lower bound of (\ref{haninInequ}) from \cite{hanin2017approximating} to a wider class of activation functions.
 
\begin{corollary}\label{lowerBoundWidth}
	For network functions with continuous and strictly monotonic activation function $\sigma$ or $\sigma=$ReLU the lower estimate $d_{in}<\omega_{\min}(\sigma,d_{in},d_{out})$ holds.
\end{corollary}
\begin{proof}
	 For the purpose to show the result by contradiction, let $f:K:=[0,1]^{d_{in}}\rightarrow \R$ be continuous with $f(x_h)=-1$ where $x_h=(1/2,...,1/2)^T$ and $f=1$ on the boundary of $K$ and assume that 
	\begin{equation} \label{approxContra}
	\left\lVert f-F\right\lVert_K < 1/2.
	\end{equation}
	Then necessarily $F(x_h)<0$ and $F(x_b)>0$ for all $x_b$ in $\partial K$.
	By Theorem \ref{theoDecRIntersect} the preimage of $(-\infty,0)$ under $F$ (the decision region $F<0$) is either empty or intersects the boundary of $K$ and hence gives a contradiction, since $F>1/2$ holds on $\partial K$.
\end{proof}

\begin{remark}\label{remarkLowerTrivial}
	It is easily seen that in general  $\omega_{\min}(\sigma,d_{in},d_{out})<d_{in}$ is impossible for all $\sigma$. Indeed, in this case $W_1$ would have non trivial kernel and therefore every function that is non-constant on all subspaces of $\R^{d_{in}}$, such as $x\mapsto \prod_{j=1}^{d_{in}} x_j^2$, cannot be approximated with arbitrary accuracy. 
\end{remark}

We now formulate an example that shows that, despite the restrictions given by Theorem \ref{theoDecRIntersect} and Theorem \ref{theHeinConnected}, decision regions can be disconnected as subset of a compact input domain (c.f. Definition \ref{innerConnect}).

\begin{example}\label{exampNotPathCon}
	Let $K=[-1,1]\times[-1,1]$ and $\sigma$ be the ReLU activation function. The weights and bias in the first layer are set as follows: $W_1$ is the rotation matrix with angle $\alpha=-\pi/4$, i.e.:
	\begin{equation}\label{rotMat}
	W_1=
	\left[ {\begin{array}{cc}
		\cos \alpha & -\sin \alpha \\
		\sin \alpha & \cos \alpha  \\
		\end{array} } \right],
	\end{equation}
	and $b_1=\sqrt{2} (1,-1/2)^T$. The parameters for the (output) second layer are as follows: $W_2= 1/\sqrt{2}\,(1,-4)^T$, $b_2=-1/4$. The complete model now is written as 
	\begin{equation}
		F:K\rightarrow \R,\, x\mapsto W_2\,\sigma(W_1 x+b_1)+b_2.
	\end{equation}
	We further set
	\begin{equation}
	A_1:K\rightarrow \R^2,\  x\mapsto \sigma(W_1 x+b_1).
	\end{equation}
	One easily verifies that $A_1(K)$ and the decision hypersurface defined by $W_2x+b_2=0$ are as depicted in Figure \ref{exampel1_mapping}. Then the region in $K$ that is mapped to $(-\infty,0)$ is not connected in $K$ as depicted in Figure \ref{exampel1_preIm}. 
	
\end{example}

\begin{figure}
	\centering
	\includegraphics[scale=0.4]{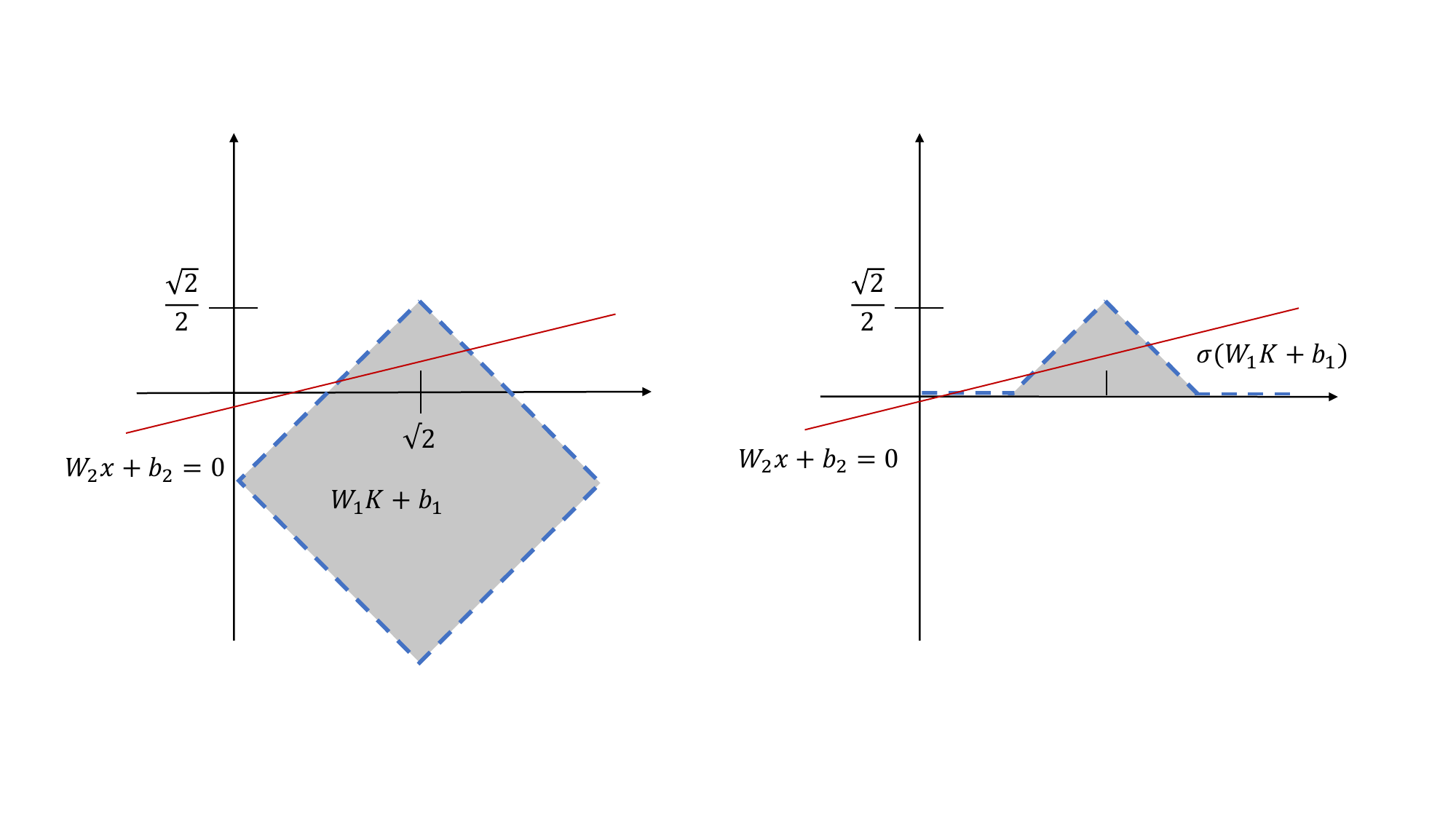}
	\caption{Example \ref{exampNotPathCon}: drawing of  $W_1K+b_1$ (left) and $\sigma(W_1K+b_1)$ right ($x_1$ horizontal, $x_2$ vertical). The red line depicts the decision hypersurface defined by $W_2\,x+b_2=0$} 
	\label{exampel1_mapping}
\end{figure}

It is straightforward to adapt the above example to the leaky ReLU activation $\sigma_\beta(t)=\max\{t,\beta \, t\},\ 0<\beta<1$. The important point is that a convex excerpt of the activation function can be used to create an internal corner in the image of $K$ under the mapping that corresponds to the first layer, c.f. Figure \ref{exampel1_mapping} on the right.

\begin{figure}
	\centering
	\includegraphics[scale=0.3]{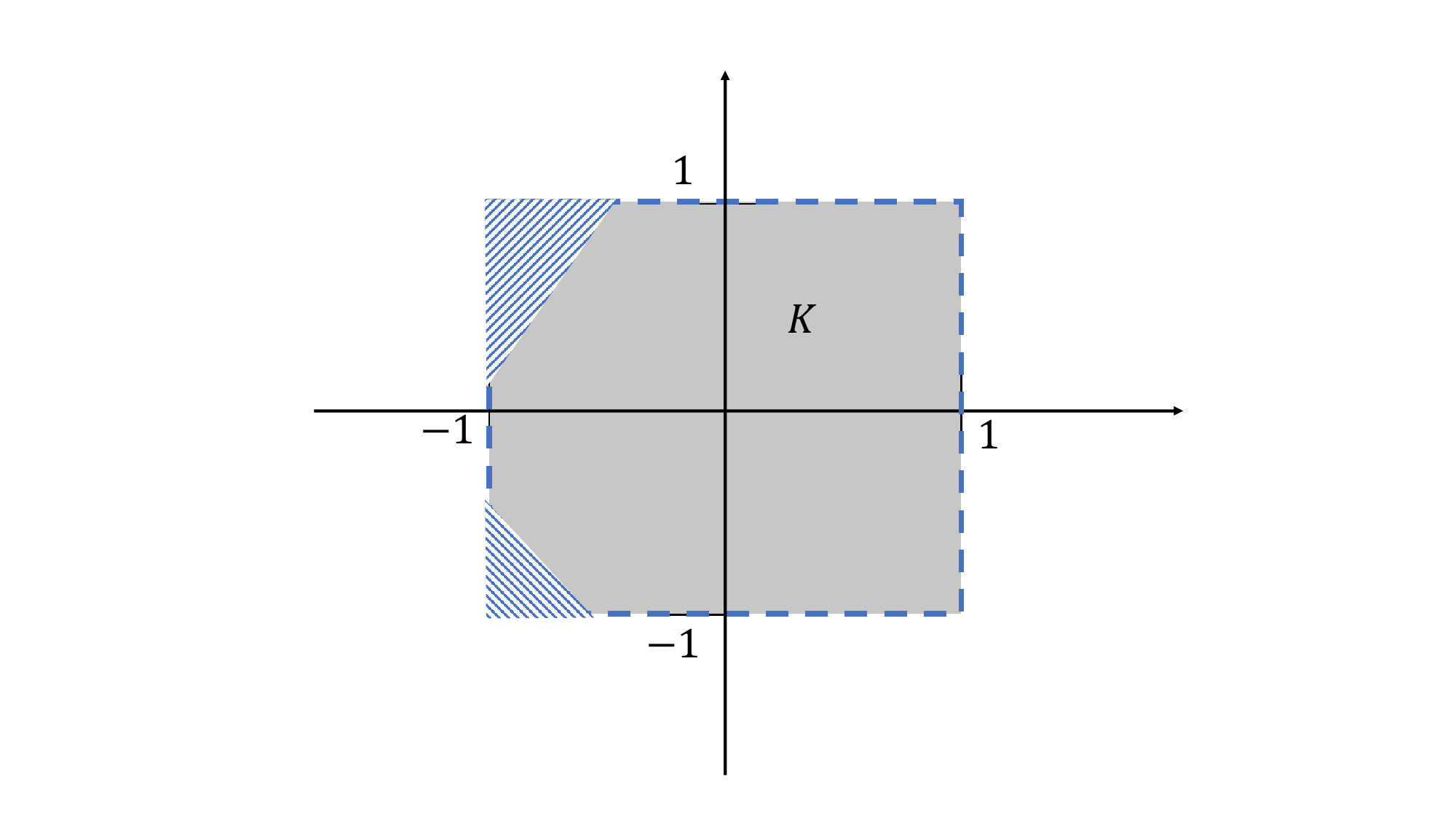} 
	\caption{Drawing of the set $K$ in Example \ref{exampNotPathCon} with crosshatched blue area for the preimage of $(-\infty,0)$ under $F$ ($x_1$ horizontal, $x_2$ vertical)}
	\label{exampel1_preIm}
\end{figure}

We further extend Example \ref{exampNotPathCon} to the case that the mapping starts at $\R^2$ rather than $K$. The example shows that the surjectivity condition Theorem \ref{theHeinConnected} cannot be dropped in general and the result does not hold for ReLU. It should be mentioned that the authors of  \cite{nguyen2018neural} are aware of this limitation as they also formulate a counter example for ReLU networks.
\begin{figure}
	\centering
	\includegraphics[scale=0.3]{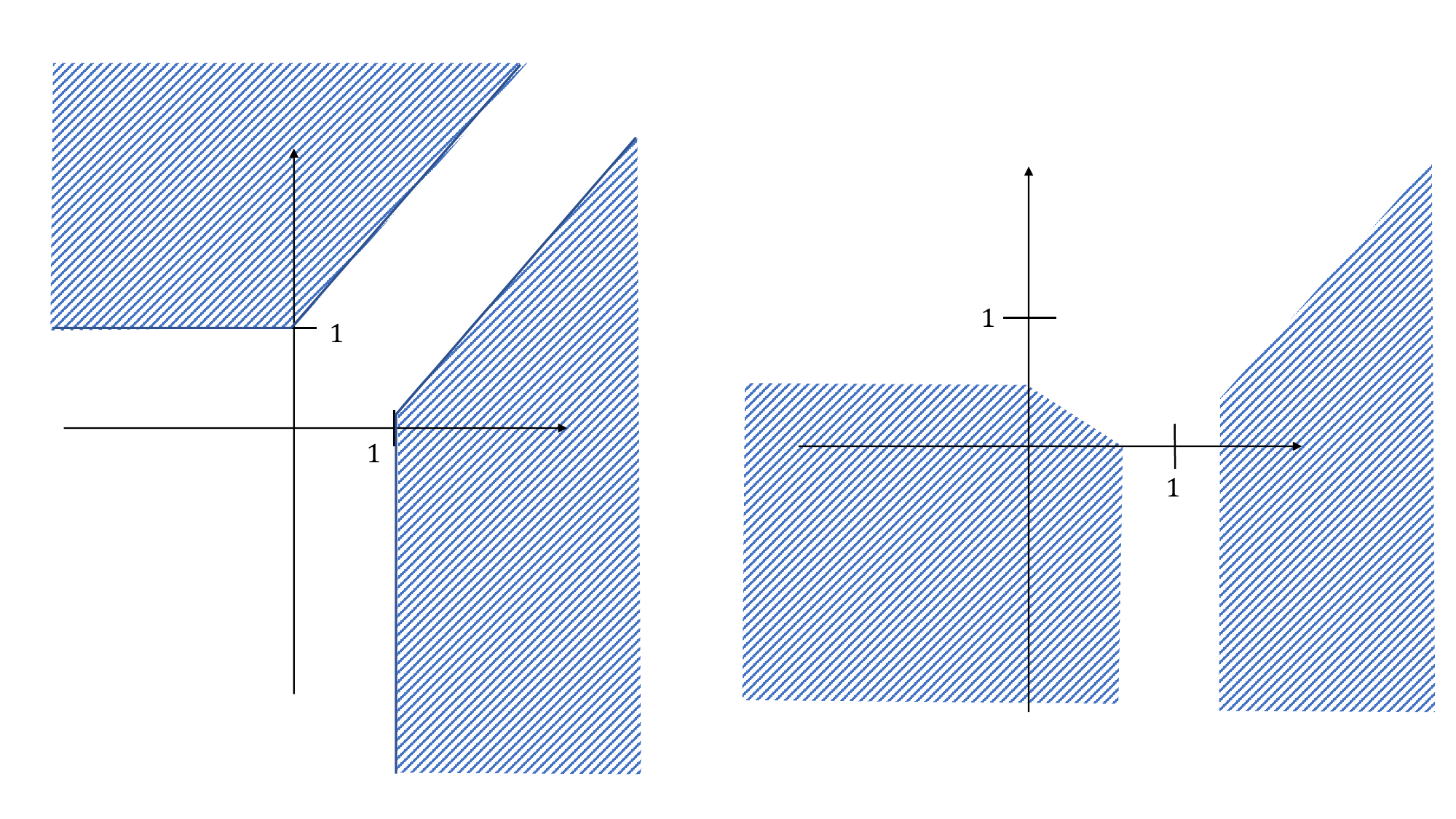} 
	\caption{Example \ref{exampNotCon} ($x_1$ horizontal, $x_2$ vertical): Left, crosshatched blue area depicts the preimage of $[0,1/\sqrt{2}]\times \{0\}\cup[3/\sqrt{2},\infty)\times \{0\}$ under $G_2$. Right: crosshatched blue area depicts the preimage of $(-\infty,0)$ under $F$}
	\label{exampel2_preIm}
\end{figure}
\begin{example}\label{exampNotCon}
	Let $\sigma$ be the ReLU activation function and $W_1=I_2$ the identity in $\R^2$ and $b_1=(0,0)^T$. Then $A_1:\R^2\rightarrow \R^2$, $x\mapsto \sigma(W_1x+b_1)$ maps $\R^2$ to $Q_1:=\{(x_1,x_2)^T\in\R^2: x_1,x_2\geq 0\}$.
	Let $W_2$ be the rotation matrix for angle $\alpha=-3/4 \pi$ (c.f. (\ref{rotMat})), and $b_2=\sqrt{2}\,(1,1/2)^T$. The resulting image of $\R^2$ under $G_2:\R^2 \rightarrow \R^2$,  $x\mapsto \sigma(W_2\,A_1(x)+b_2) $ is similar to Figure \ref{exampel1_mapping} right, except that the right bar is continued to $+\infty$. One verifies that the preimage of these bars is as depicted in Figure \ref{exampel2_preIm} left. Following Example \ref{exampNotPathCon}, we set  $W_3= 1/\sqrt{2}\,(1,-4)^T$, $b_3=1/4$ and define $F:\R^2\rightarrow \R$ as $x\mapsto W_3 G_2(x)+b_3$. Now the preimage of $(-\infty,0)$ under $F$ is not connected in $\R^2$ as sketched in Figure \ref{exampel2_preIm} right.	
\end{example}
The conclusion from the previous examples can be summarized as follows.
\begin{corollary}
From Examples \ref{exampNotPathCon} and \ref{exampNotCon} it immediately follows that in general
\begin{enumerate}
	\item for a network function with $ReLU$ or leaky $ReLU$ activation function and width not exceeding the input dimension, the decision areas are not necessarily connected with respect to a bounded input domain (c.f. Definition \ref{innerConnect}). 
	\item In Theorem \ref{theHeinConnected} the condition that the activation function maps $\R$ surjectively to $\R$ can in general not be dropped. In particular, Theorem \ref{theHeinConnected} doesn't hold for  $ReLU$ activation.
\end{enumerate}
\end{corollary}

\section{Experiments}

\begin{figure}
	\centering
	\includegraphics[width=0.460\textwidth]{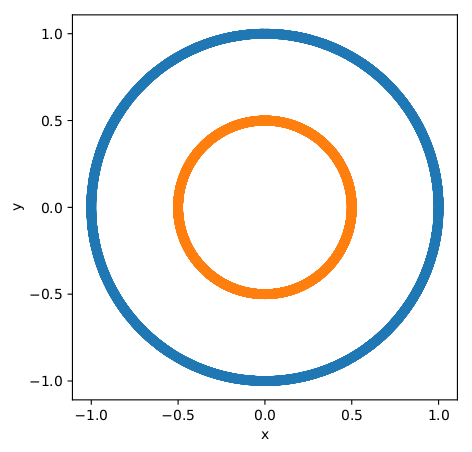}
	\hfill
	\includegraphics[width=0.490\textwidth]{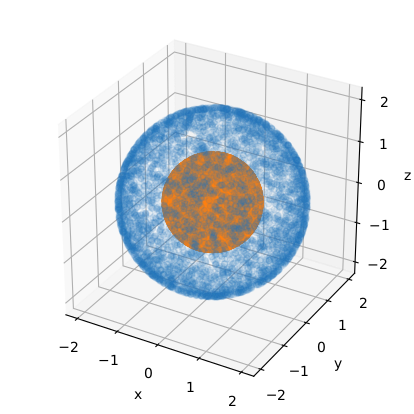}
	\caption{Input datasets in the two and three dimensional case. The dataset consists of two classes: the inner sphere (orange) labeled as $0$ and the outer sphere (blue) labeled as $1$. Left: 1-sphere with radius $0.5$ for the inner sphere and radius $1.0$ for the outer sphere. Right: 2-sphere with radius $1.0$ for the inner sphere and radius $2.0$ for the outer sphere.}
	\label{sphere_dataset}
\end{figure}
\begin{figure}
	\centering
	\begin{subfigure}{0.19\textwidth}
         \centering
         \includegraphics[width=\textwidth]{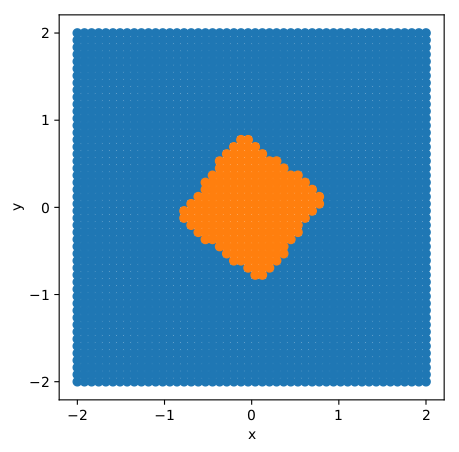}
         \caption{Success - Epoch 10}
    \end{subfigure}
	\hfill
	\begin{subfigure}{0.19\textwidth}
         \centering
         \includegraphics[width=\textwidth]{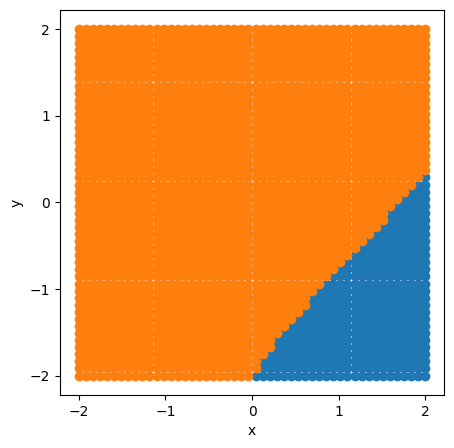}
         \caption{Failure - Epoch 0}
    \end{subfigure}
    \hfill
	\begin{subfigure}{0.19\textwidth}
         \centering
         \includegraphics[width=\textwidth]{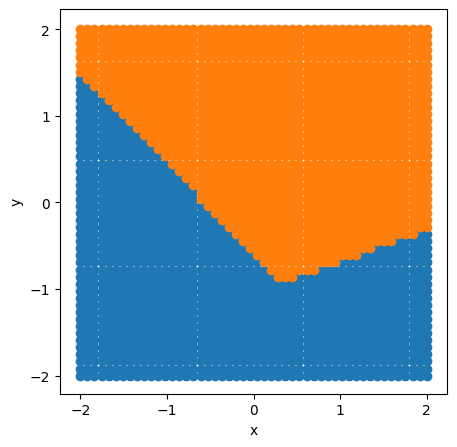}
         \caption{Failure - Epoch 1}
    \end{subfigure}
    \hfill
	\begin{subfigure}{0.19\textwidth}
         \centering
         \includegraphics[width=\textwidth]{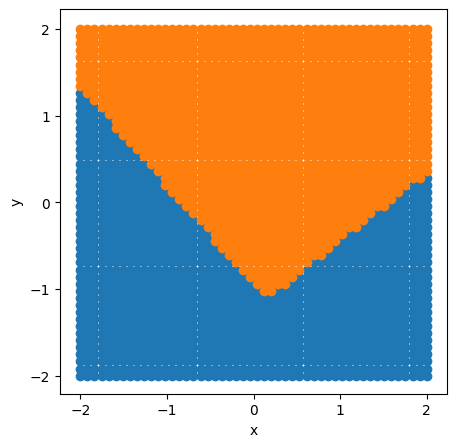}
         \caption{Failure - Epoch 2}
    \end{subfigure}
    \hfill
	\begin{subfigure}{0.19\textwidth}
         \centering
         \includegraphics[width=\textwidth]{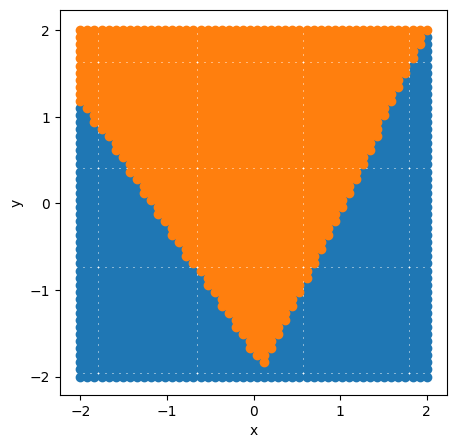}
         \caption{Failure - Epoch 10}
    \end{subfigure}
	\caption{Decision regions by trained networks for input dimension two. (a): Decision regions for a network with one hidden layer of size $d_{in}+1$ obtaining $100$\% test accuracy. (b)-(e): Decision regions for a network with one hidden layer of size $d_{in}$ over several epochs. The results show that the decision regions need to be unbounded throughout training even before the first epoch, i.e. at initialization.}
	\label{sphere_decision_dim2}
\end{figure}
\begin{figure}
	\centering
	\begin{subfigure}{0.205\textwidth}
         \centering
         \includegraphics[width=\textwidth]{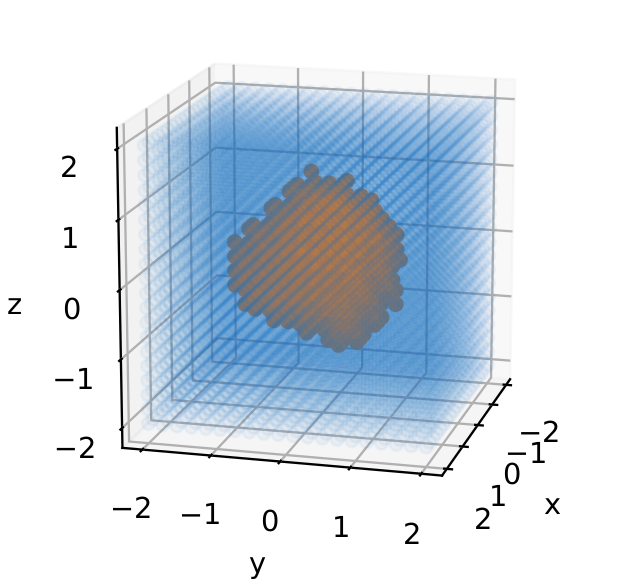}
         \caption{Success - Epoch 10}
    \end{subfigure}
	\hfill
	\begin{subfigure}{0.19\textwidth}
         \centering
         \includegraphics[width=\textwidth]{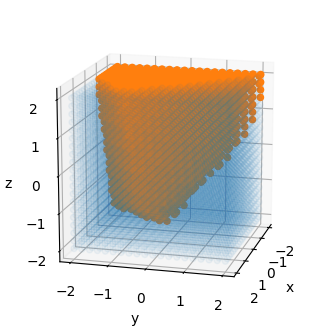}
         \caption{Failure - Epoch 0}
    \end{subfigure}
    \hfill
	\begin{subfigure}{0.19\textwidth}
         \centering
         \includegraphics[width=\textwidth]{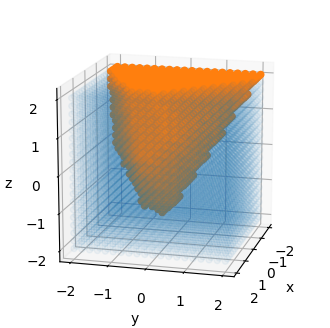}
         \caption{Failure - Epoch 1}
    \end{subfigure}
    \hfill
	\begin{subfigure}{0.19\textwidth}
         \centering
         \includegraphics[width=\textwidth]{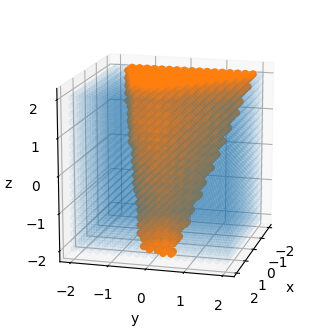}
         \caption{Failure - Epoch 2}
    \end{subfigure}
    \hfill
	\begin{subfigure}{0.19\textwidth}
         \centering
         \includegraphics[width=\textwidth]{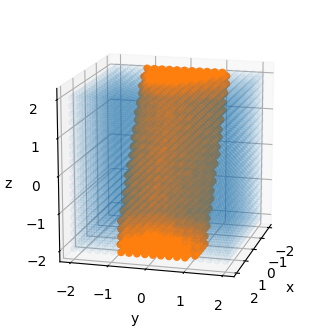}
         \caption{Failure - Epoch 10}
    \end{subfigure}
	\caption{Decision regions by trained networks for input dimension three. (a): Decision regions for a network with one hidden layer of size $d_{in}+1$ obtaining $100$\% test accuracy. (b)-(e): Decision region for a network with one hidden layer of size $d_{in}$ over several epochs. The results show that the decision regions need to be unbounded throughout the training even before the first epoch, i.e. at initialization.}
	\label{sphere_decision_dim3}
\end{figure}
\begin{figure}
	\centering
	\includegraphics[width=0.475\textwidth]{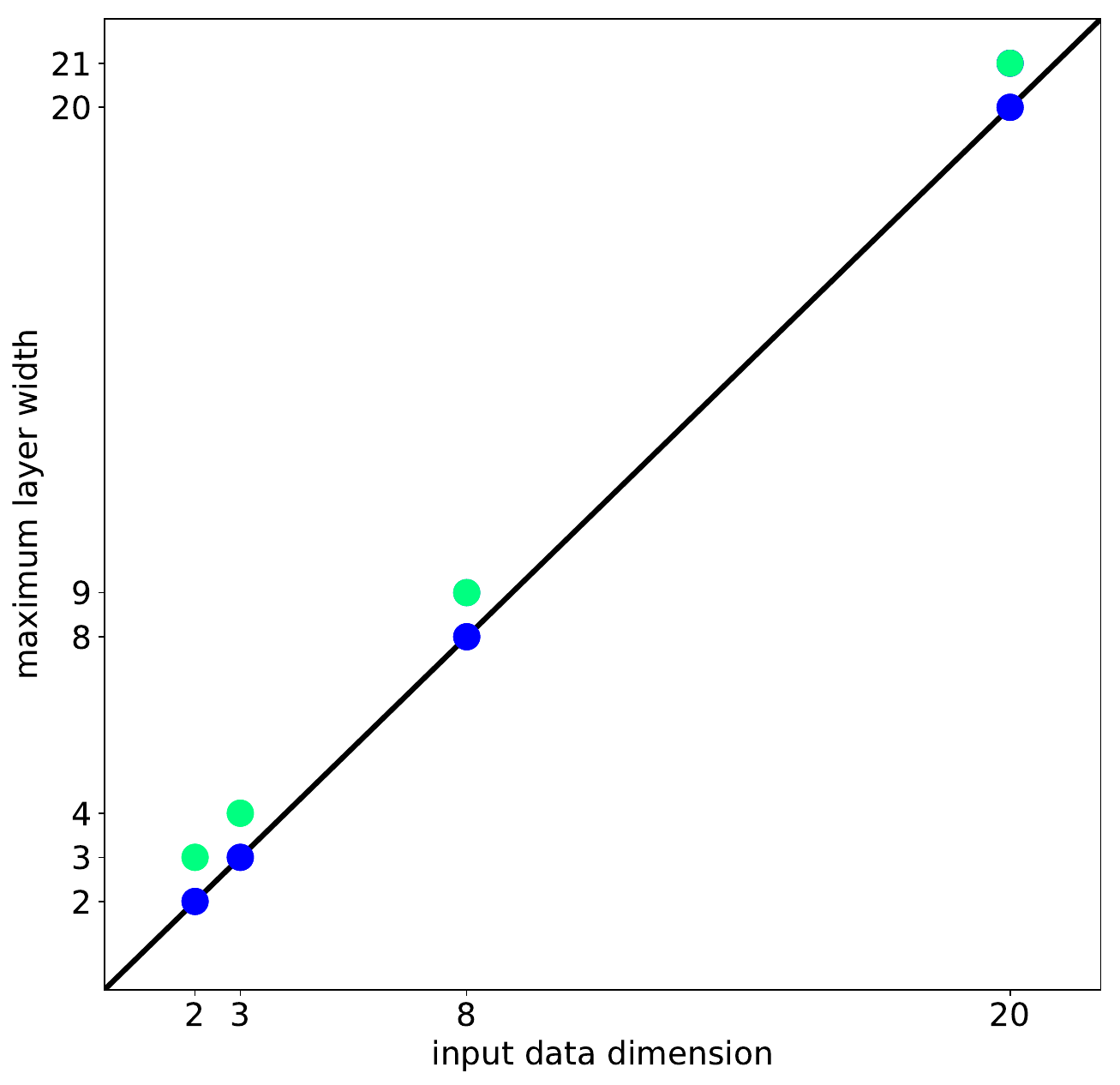}
	\hfill
	\includegraphics[width=0.475\textwidth]{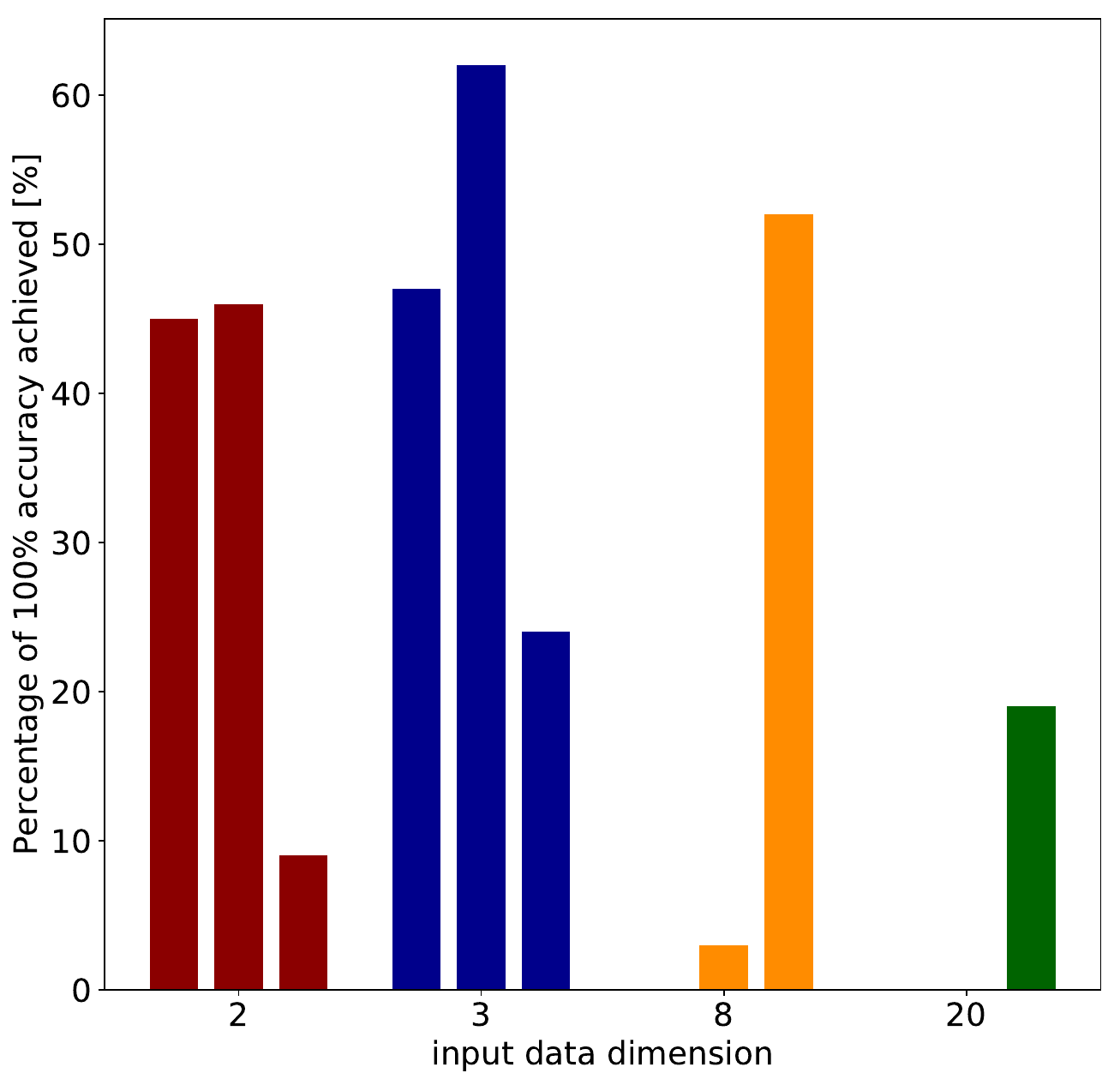}
	\caption{Results of the complete test run (24 combination of settings trained $100$ times for $10$ epochs). Left: Input data dimension $d_{in}$ plotted against maximal network width $\omega_F$. Green points are instances for which the network obtained at least once $100$\% accuracy on the test set. Blue points are instances for which the network never obtained $100$\% test accuracy. We clearly see the trend that the network is only capable of learning correct decision regions when $\omega_F>d_{in}$. Right: For each input data dimension (2, 3, 8 and 20), we tested three different network sizes (colored in the same color from left to right: 1,2 and 10). The height of the bar plot represents the percentage of networks which obtained $100$\% accuracy on the test set out of 100 repeated trainings. For $d_{in}=8$ and $d_{in}=20$ some of the network settings did never achieve $100$\% although success should be possible. That is why some bars vanish completely. We only plot networks for which $\omega_F>d_{in}$, since otherwise the percentage is zero.}
	\label{sphere_stats}
\end{figure}
In this section, we present the results of some numerical experiments. In the first instance, those experiments where carried out for illustrations purposes. By means of simple artificial data, we show that the limitations that we derived in our main result can be observed in low dimensional setting. 
Secondly, we briefly investigated if our theoretically derived limitation affect the capabilities of neural networks to approximate common (simple) training data sets. 

 We defined a dataset consisting of two n-spheres both being centered at the origin. The inner n-sphere (first class) has a radius of $\frac{n}{2}$ and the outer n-sphere (second class) a radius of $n$. We increase the radius of the n-sphere with the dimension in order to avoid numerical problems when we go to higher dimensions. The two and three dimensional dataset is shown in Fig. \ref{sphere_dataset}. The neural network we wanted to train should separate the inner sphere from the outer sphere and we considered a trained network to be successsful only when we reached a test accuracy of $100$\%. According to Theorem \ref{theoDecRIntersect}, this should only be possible if $\omega_F > d_{in}$. If $\omega_F \leq d_{in}$, we should always get an unbounded decision region and consequently not obtain a test accuracy of $100$\%. We generated $10^7$ uniformly distributed points on each sphere for the training data and $25\cdot 10^5$ uniformly distributed points on each sphere for the test data. Each network was trained for $10$ epochs using a batch size of $10^4$ and the Adam optimizer with a learning rate of $0.001$. For the cost function we used the cross entropy. We trained neural networks of different sizes with fully connected layers and ReLU activation functions only. We combined the following different possible settings:
\begin{enumerate}
	\item Input dimension: 2, 3, 8 and 20,
	\item Number of fully connected layers: 1, 2, 10,
	\item All layers of width $d_{in}$ and layers of width $d_{in}+1$.
\end{enumerate}
Each one of the $24$ possible combinations of settings was repeated for $100$ trainings of $10$ epochs in order to have a meaningful representation of the setting. For each training, we stored the maximal obtained accuracy on the test dataset as the performance of the training. For input data of dimension two and three, Fig. \ref{sphere_decision_dim2}, and Fig. \ref{sphere_decision_dim3}, respectively, illustrate the successful training of a network with $\omega_F > d_{in}$ and an example of an unsuccessful training of a network with $\omega_F \leq d_{in}$. As expected, the networks with $\omega_F \leq d_{in}$ learn an unbounded decision region. Moreover, the results show that unbounded decision regions cannot be avoided throughout the training, even at the network weight initialization. This highlights that the results obtained in Theorem \ref{theoDecRIntersect} are algorithm-independent. The summary of the complete parameter study is illustrated in Fig. \ref{sphere_stats}. Successful training was only possible when $\omega_F > d_{in}$, although some of the settings which could have been successful did not achieve $100$\% test accuracy.

\begin{figure}
	\centering
	\includegraphics[width=0.93\textwidth]{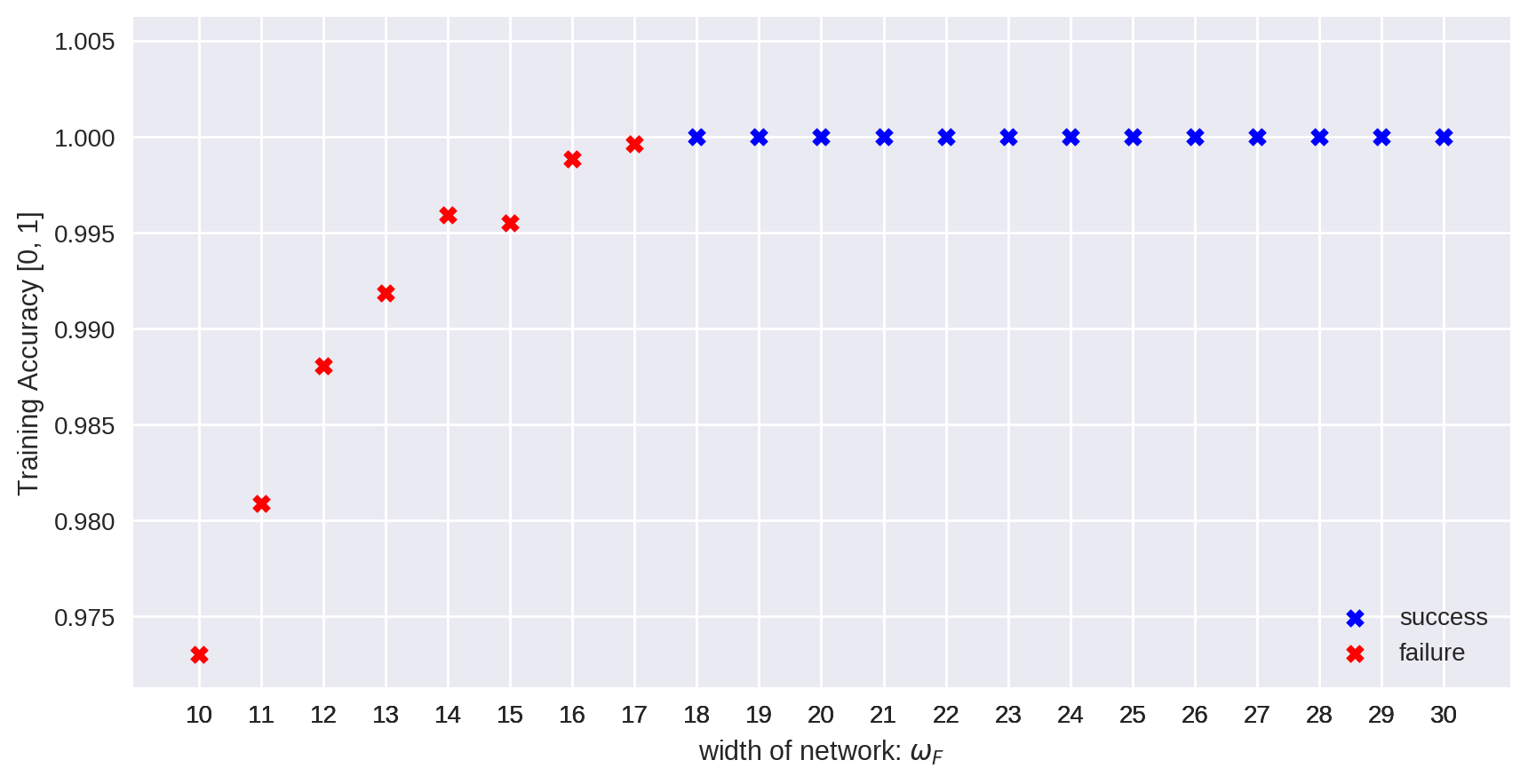} 
	\caption{Results of our investigations on achieving 100\% accuracy on the training set for MNIST  as a function of the width $\omega_F$ of the neural networks: In blue we mark the widths for which our experiments yielded 100\% accuracy whereas in red we mark the widths for which we could not train a model to classify all training samples correctly. The dimension of the input data is $28*28=784.$}
	\label{mnist}
\end{figure}

\begin{figure}
	\centering
	\includegraphics[width=0.93\textwidth]{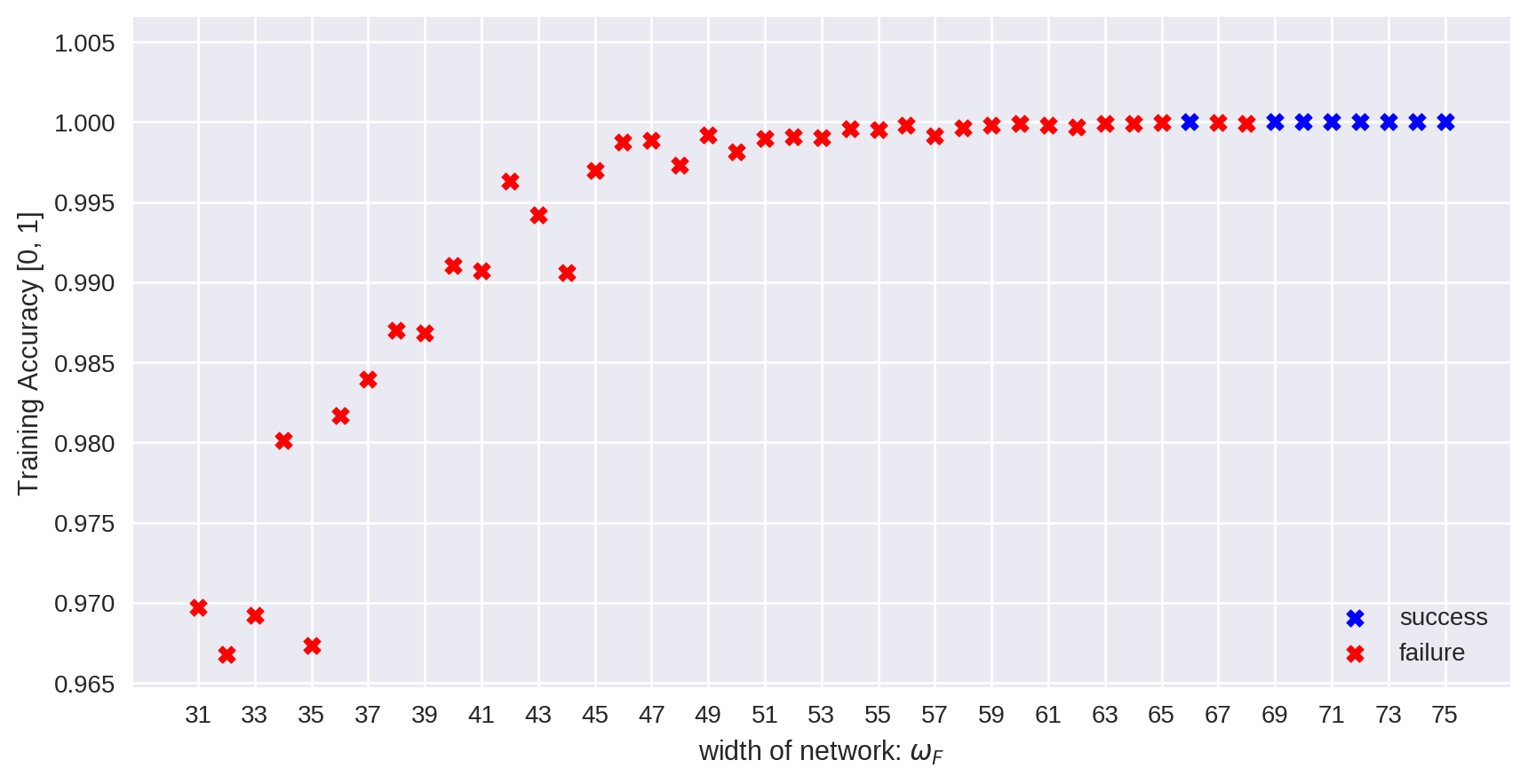} 
	\caption{Results of our investigations on achieving 100\% accuracy on the training set for Fashion MNIST  as a function of the width $\omega_F$ of the neural networks: In blue we mark the widths for which our experiments yielded 100\% accuracy whereas in red we mark the widths for which we could not train a model to classify all training samples correctly. The dimension of the input data is $28*28=784.$}
	\label{fashion}
\end{figure}

\begin{figure}
	\centering
	\includegraphics[width=0.93\textwidth]{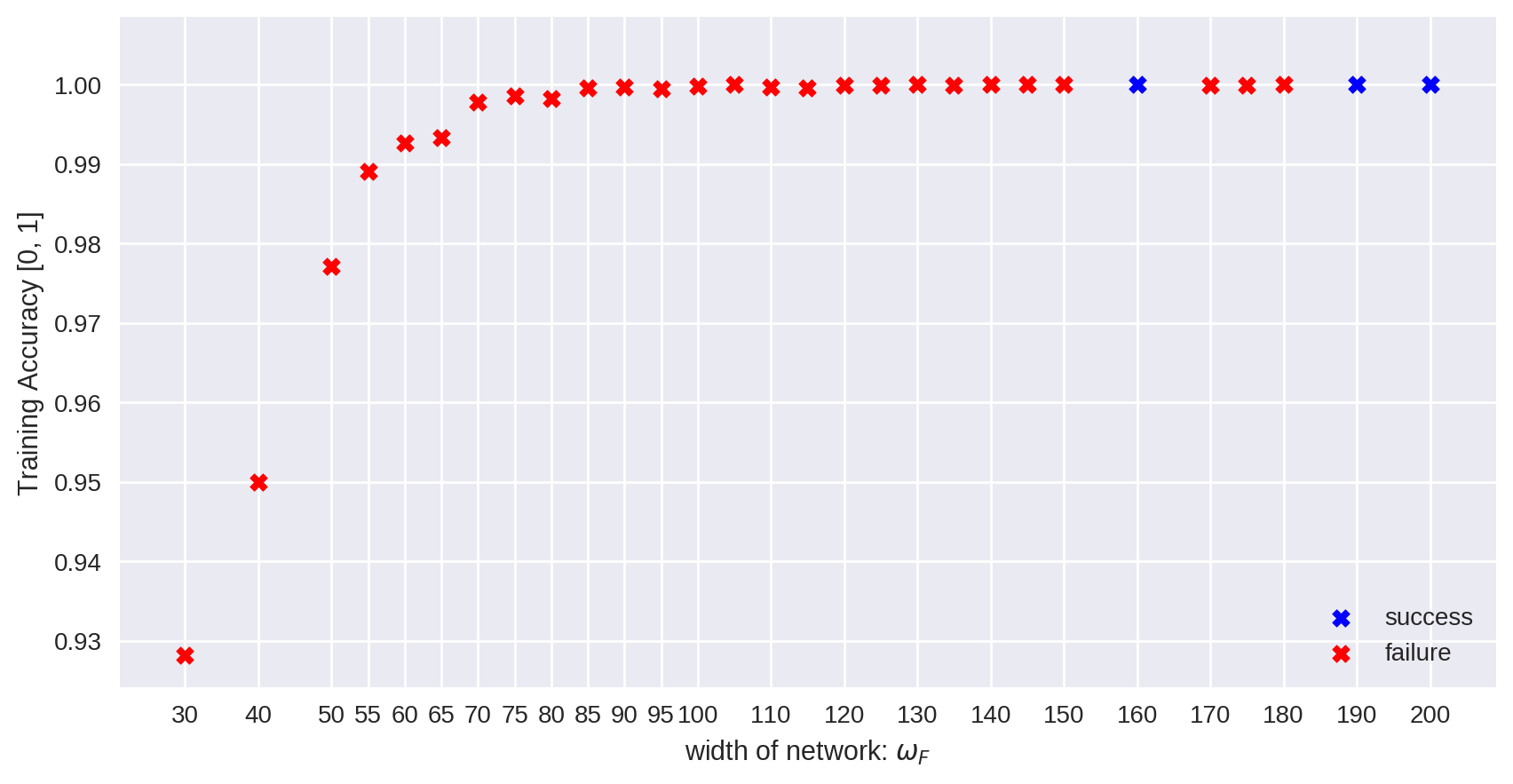} 
	\caption{Results of our investigations on achieving 100\% accuracy on the training set for EMNIST  as a function of the width $\omega_F$ of the neural networks: In blue we mark the widths for which our experiments yielded 100\% accuracy whereas in red we mark the widths for which we could not train a model to classify all training samples correctly. The dimension of the input data is $28*28=784.$}
	\label{emnist}
\end{figure}

In our second experiment we wanted to test the applicability of our results on widely used benchmark datasets. The goal of this experiment was to achieve 100\% accuracy on the training datasets for a varying width $\omega_F$ of the neural network. When the expressiveness of the neural network is high enough, then we should be able to classify all samples on the training distribution correctly. On the other side, once the width $\omega_F$ of the neural network is too small, we should no longer be able to classify all samples correctly. We did not consider the test dataset in this setting, because our work is about approximation capabilities and hence does not consider generalization. This is not a problem for the first experiment, as in the former the manifold of the data is well understood and train and test samples behave similarly. Our investigations were performed as follows: We trained many different neural networks until we reached 100\% accuracy on the training dataset or we stopped the training after 500 epochs. We tested different optimizers (Adam, SGD with momentum and Nesterov, and RMSProp), different learning rates, different batch sizes and we alternatively included a scheduler which decreased the learning rate once a plateau was reached. We did not use weight decay or any other regularization techniques. We used neural networks consisting of fully connected layers only with 2 up to 5 hidden layers. The width of all layers was chosen to be the same. However, we did not consider neural networks with a width smaller than the number of classes of the dataset used. We tested neural networks mostly using ReLU as an activation function, but did also experiments with tanh, however, both behaved similarly. We tested our approach on MNIST and Fashion MNIST (10 classes) as well as on the EMNIST letters (26 classes) using the cross-entropy as the cost function. From all our experiments, we report in Fig. \ref{mnist} (MNIST), Fig. \ref{fashion} (Fashion MNIST) and Fig. \ref{emnist} (EMNIST letters) in blue (success) all the widths for which we achieved 100\% accuracy on the training distribution and mark by red (failure) all the widths for which this was not achieved in our experiments. Achieving 100\% accuracy consistently was more challenging then we anticipated and for some widths we did not achieve 100\% accuracy although it was achieved for smaller widths. We report only the transition phase between successful and failed widths as the 100\% training accuracy was already achievable for widths much smaller than the input dimension of the dataset. However, our results show that the trend in general is the same for all datasets: on the one side the expressiveness decreases drastically when a certain threshold is reached and on the other side, in contrast to our theoretically derived limitations, the input dimension limitation does not provide any restriction for widely used benchmark datasets. 

It is a common belief that the true dimension of a dataset does not correspond to the dimension of the input space (e.g. pixel space for images), but that the dataset can be represented by a lower dimensional (non-linear) manifold. The dimension of the latter is commonly called the intrinsic dimension, which is for most datasets unknown, but its determination for example is investigated in \cite{levina2005maximum}. Considering our results on the benchmark datasets, we would like to raise the question, and leave this as an open question for future work, whether the theoretical results as derived in the beginning of our work could show limitations that would become relevant for real data in their lower dimensional representation corresponding to their intrinsic dimension.

\section{Conclusion}

For a wide class of activation functions, we have shown that for neural network functions that have a maximum width less than or equal to the input dimension the connected components of the decision regions are unbounded. Hence, for such networks the decision regions intersect the boundary of a natural input domain. This links some recent results from \cite{hanin2017approximating} and \cite{nguyen2018neural}, where for such narrow neural networks limitations regarding their expressive power for the case of ReLU activation are achieved in the first work, and the connectivity of decision regions is restricted for the case of continuous, monotonically increasing and surjective $\sigma:\R\rightarrow \R$ in the second work. We illustrated our findings by numerical experiments with spherical data where it was observed that the input dimension is the critical threshold for network width in order to achieve $100$\% accuracy. However, in experiments on MNIST, Fashion MNIST and EMNIST letters we could not detect a limitation on the performance for such narrow networks. This raises the question to what extent limitations in terms of connectivity imply a crucial restriction in practical applications. From theoretical perspective, it would be interesting to know if Theorem \ref{theoDecRIntersect} still holds for other types of activation function, like non-continuous and oscillating activations and whether the restriction of the input dimension can be relaxed to the dimension of the underlying data manifold.

\section*{Acknowledgment}
This work was partially supported by the MECO project ”Artificial Intelligence for Safety Critical Complex Systems”. The second author is supported by the Luxembourg National Research Fund (FNR) under the grant number 13043281.

\newpage

\bibliographystyle{alpha}
\bibliography{bib/bib_nn_02}

\end{document}